\documentclass{article}

\usepackage{microtype}
\usepackage{graphicx}
\usepackage{subfigure}
\usepackage{booktabs} % for professional tables

\usepackage{hyperref}

\usepackage{tikz}
\usetikzlibrary{arrows.meta,positioning,calc}

\usepackage[preprint]{icml2026}

\usepackage{amsmath}
\usepackage{amssymb}
\usepackage{mathtools}
\usepackage{amsthm}

\usepackage[capitalize,noabbrev]{cleveref}

\usepackage{xcolor}

\theoremstyle{plain}
\newtheorem{theorem}{Theorem}[section]
\newtheorem{proposition}[theorem]{Proposition}
\newtheorem{lemma}[theorem]{Lemma}

\theoremstyle{definition}
\newtheorem{definition}[theorem]{Definition}
\newtheorem{assumption}[theorem]{Assumption}
\theoremstyle{remark}

\usepackage[textsize=tiny]{todonotes}

\icmltitlerunning{Monotonicity as an Architectural Bias for Robust Language Models}

\begin{document}

\twocolumn[
\icmltitle{Monotonicity as an Architectural Bias for Robust Language Models}



\icmlsetsymbol{equal}{*}

\begin{icmlauthorlist}
\icmlauthor{Patrick Cooper}{cu}
\icmlauthor{Alireza Nadali}{cu}
\icmlauthor{Ashutosh Trivedi}{cu}
\icmlauthor{Alvaro Velasquez}{cu}
\end{icmlauthorlist}

% Affiliation
\icmlaffiliation{cu}{Department of Computer Science, University of Colorado Boulder, Boulder, CO, USA}

% Corresponding Author(s)
% usually only one or two are listed, but you can swap these out as needed
\icmlcorrespondingauthor{Patrick Cooper}{patrick.cooper@colorado.edu}
\icmlcorrespondingauthor{Alireza Nadali}{alireza.nadali@colorado.edu}

\icmlkeywords{Machine Learning, ICML}

\vskip 0.3in
]



\printAffiliationsAndNotice{\icmlEqualContribution} % otherwise use the standard text.

\begin{abstract}

Large language models (LLMs) are known to exhibit brittle behavior under adversarial prompts and jailbreak attacks, even after extensive alignment and fine-tuning. This fragility reflects a broader challenge of modern neural language models: small, carefully structured perturbations in high-dimensional input spaces can induce large and unpredictable changes in internal semantic representations and output.

We investigate \emph{monotonicity} as an architectural inductive bias for improving the robustness of Transformer-based language models. Monotonicity constrains semantic transformations so that strengthening information, evidence, or constraints cannot lead to regressions in the corresponding internal representations. Such order-preserving behavior has long been exploited in control and safety-critical systems to simplify reasoning and improve robustness, but has traditionally been viewed as incompatible with the expressivity required by neural language models.

We show that this trade-off is not inherent. By enforcing monotonicity selectively in the feed-forward sublayers of sequence-to-sequence Transformers—while leaving attention mechanisms unconstrained—we obtain \emph{monotone language models} that preserve the performance of their pretrained counterparts. This architectural separation allows negation, contradiction, and contextual interactions to be introduced explicitly through attention, while ensuring that subsequent semantic refinement is order-preserving. Empirically, monotonicity substantially improves robustness: adversarial attack success rates drop from approximately 69\% to 19\%, while standard summarization performance degrades only marginally.

% Our results demonstrate that strong architectural constraints, long assumed to limit the capacity of LLMs, can be applied at scale without sacrificing accuracy. Monotonicity thus emerges as a practical and principled design choice for building language models whose semantic behavior is more robust, predictable, and amenable to formal reasoning.
\end{abstract}

\section{Introduction}

Large language models (LLMs) have demonstrated remarkable capabilities across a wide range of natural language tasks, yet their behavior under adversarial or carefully structured inputs remains brittle. Even models that have undergone extensive alignment and fine-tuning can be induced to produce unsafe or unintended outputs through jailbreak prompts or small, targeted perturbations. This phenomenon is now well documented and points to a deeper underlying issue: modern language models operate in extremely high-dimensional spaces, where even subtle changes to the input or internal representations can induce large and difficult-to-predict changes in the output. Recent work has attributed this sensitivity, in part, to the high Lipschitz constants exhibited by such models, which amplify perturbations as they propagate through the network~\cite{newhouse2025training}.

This fragility has motivated a line of research on improving the robustness of language models~\cite{salah2026jailbreaking, su2024mission}. Existing approaches largely focus on training-time interventions, such as alignment objectives~\cite{cao2024defending} or adversarial training~\cite{howe2024effects}, as well as inference-time defenses such as output filtering~\cite{khachaturov2025adversarial}. While these methods can be effective empirically, they are inherently reactive, addressing vulnerabilities after they are discovered rather than constraining model behavior by design. As language models are increasingly deployed, this raises a question: can we endow language models with structural properties that make their behavior more predictable under perturbation?

We investigate \emph{monotonicity} as a structural inductive bias for improving the robustness of language models. At a high level, monotonicity constrains a model so that its outputs change in a predictable direction with respect to changes in its inputs. This property has a long history in machine learning~\cite{Gupta2016MonotonicLattices, runje2023constrained} and related fields~\cite{smith1995monotone}, and has been especially influential in control theory, where monotone systems admit strong guarantees on stability and safety even in very high-dimensional settings~\cite{michel2015stability,ito2014max}. For example, monotonicity has been used to analyze and certify the behavior of large-scale dynamical systems such as power grids, where direct reasoning over all possible system states is infeasible~\cite{rantzer2014control}.

Beyond its role in classical dynamical systems, monotonicity also aligns naturally with how we expect language models to behave at the level of semantic transformation. Many linguistic and reasoning operations are inherently order-preserving: adding supporting evidence should not weaken a conclusion; adding constraints or safety instructions should not result in less constrained internal representations; and increasing the stated severity or risk of a situation should not lead to a more permissive response. Yet contemporary language models frequently violate these expectations, exhibiting behavior in which the addition of seemingly clarifying or restrictive information leads to unpredictable or regressive changes in output.

Concrete examples of semantic monotonicity arise naturally in language understanding. 
Consider a sequence of prompts that progressively strengthen a factual claim (\emph{``The medication has side effects''} $\preceq$ \emph{``The medication has severe side effects''} $\preceq$ \emph{``The medication has severe side effects and requires medical supervision''}), or that add explicit safety constraints (\emph{``Explain this process''} $\preceq$ \emph{``Explain this process safely''} $\preceq$ \emph{``Explain this process safely without providing actionable steps''}). Semantically, such prompts form a partial order: later prompts contain strictly more information or stricter constraints. A model's internal representations should reflect this ordering, and its outputs should respect it. When models fail to do so, the cause is often not the absence of relevant information, but the way semantic features are transformed and recombined internally.

This perspective highlights a structural source of brittleness in modern architectures. Transformer models interleave contextual aggregation, via attention, with nonlinear feed-forward transformations that refine token-level representations. While attention mechanisms are well suited to introducing contextual operators such as negation, contrast, or exception by explicitly linking tokens and defining scope, the subsequent feed-forward transformations are unconstrained and may arbitrarily amplify, suppress, or cancel features. As a result, semantic reversals can occur implicitly through distributed feature cancellation rather than being driven by explicit linguistic signals. We argue that this implicit cancellation is a key contributor to brittle behavior under adversarial or carefully structured prompts.

Our motivation for monotonicity draws a close conceptual parallel to \emph{monotone Boolean circuits}~\cite{jukna2012boolean}. In monotone circuits, logical negation is disallowed within the circuit itself; any negation must be introduced explicitly at the inputs. This restriction does not eliminate expressive power, but enforces a semantic discipline: truth cannot be undone by downstream computation. Analogously, enforcing monotonicity in semantic transformations encourages language models to represent negation, contradiction, or constraint relaxation explicitly through contextual mechanisms such as attention, rather than implicitly through fragile cancellations inside nonlinear transformations. Monotonicity thus serves as a design principle that separates \emph{where semantic reversals may occur} from \emph{where meaning is refined}, yielding predictable, auditable, and robust representations.

Viewed through this lens, a language model can be understood as a high-dimensional system whose internal semantic state evolves through successive transformations in response to input perturbations. Imposing monotonic structure on these transformations limits how adversarial variations propagate, ensuring that strengthening information or constraints cannot lead to regressions in meaning. As in monotone dynamical systems, this structure simplifies reasoning about global behavior. In such settings, it suffices to verify satisfaction only at extreme points~\cite{feyzmahdavian2017stability}.

Our central finding is that monotonicity can be incorporated into Transformer architectures without sacrificing performance by enforcing it selectively within feed-forward sublayers, while leaving attention mechanisms unconstrained. This design preserves expressivity while imposing a disciplined form of semantic refinement. The resulting monotone Transformers match the task performance of their non-monotone counterparts and exhibit substantially improved robustness to adversarial and jailbreak attacks. 
In particular, these gains arise purely from architectural structure, without additional data, modified training objectives, or heuristic defenses, isolating monotonicity itself as a causal factor.

%%--- move the following to the main body ----
% Our central finding is that monotonicity can be incorporated into modern Transformer architectures without sacrificing performance. Rather than constraining the entire model, we apply monotonicity selectively to the feed-forward sublayers that dominate the model’s nonlinear transformations, while leaving attention mechanisms, residual connections, and normalization layers unconstrained. To enable effective optimization under these constraints, we adopt a smooth reparameterization that avoids the degeneracies associated with naive projection-based approaches. This design allows pretrained sequence-to-sequence models to be distilled into monotone counterparts that closely match their original performance on standard summarization benchmarks.

% The resulting models exhibit a striking empirical property: despite preserving task accuracy, monotone Transformers are substantially more robust to adversarial and jailbreak attacks. These improvements are not achieved through additional data, modified training objectives, or heuristic defenses, but emerge purely from architectural structure. This isolation makes the observed robustness particularly notable, as it identifies monotonicity itself (rather than confounding training choices) as a causal factor.

These results challenge the prevailing view, articulated in prior studies of constrained monotonic networks~\cite{sartor2025advancing}, that enforcing strong architectural constraints inevitably degrades optimization or expressive capacity in large neural models. Instead, they suggest that monotonicity can serve as a practical and principled inductive bias for shaping semantic behavior. Motivated by this perspective, we investigate how monotonicity can be imposed as a structural property of language models in a manner that is both semantically meaningful and practically viable. In particular, we focus on enforcing monotonicity at the level of semantic refinement, rather than as a global constraint, and develop this approach through both formal analysis and empirical evaluation.

% Taken together, these results highlight two broader implications. First, they demonstrate that strong architectural constraints, long assumed to be impractical for large language models, can in fact be applied at scale without noticeable degrading performance. Second, they suggest that robustness in language models need not rely exclusively on increasingly complex training pipelines, but can instead be supported by principled design choices that shape model behavior by construction. Monotonicity thus emerges as a practical and scalable inductive bias for building more robust and predictable language models.

\noindent\textbf{Contributions.}
Our contributions are as follows:
\vspace{-1em}
\begin{enumerate}
\item \textbf{Monotone Transformer components.} We introduce monotone feed-forward sublayers for Transformers, enforcing structural monotonicity while leaving attention mechanisms unconstrained.
% \vspace{-0.5em}
\item \textbf{Lossless distillation.} We show that pretrained Transformers can be distilled into monotone counterparts without degrading standard benchmark performance.
% \vspace{-0.5em}
\item \textbf{Improved robustness.} 
We empirically demonstrate that monotone language models exhibit increased robustness to adversarial and jailbreak attacks.

\end{enumerate}

\section{Related Work}

\textbf{Vulnerabilities of Language Models.}
Large language models are highly susceptible to adversarial manipulations that bypass alignment and safety mechanisms. Prior work has documented jailbreak prompts and transferable adversarial suffixes that generalize across models and settings, including aligned systems~\citep{Zou2023UniversalTransferableAttacksOnAlignedLLMs}. Subsequent studies identify additional failure modes, such as position-sensitive jailbreaks~\citep{wang2025vulnerability}, interpretable attacks that evade perplexity-based defenses~\citep{Zhu2023AutoDAN}, and many-shot jailbreaks whose effectiveness increases with larger context windows~\citep{anil2024many}. Robustness benchmarks further show that models perform poorly under systematic adversarial evaluation~\citep{Wang2021AdvGLUE}, and that safety mechanisms often fail in multilingual or low-resource settings~\citep{deng2024multilingual}. Prompt injection attacks, which exploit the lack of separation between instruction and context, are now recognized as a fundamental vulnerability, alongside broader threats to inference and training-time surveyed in recent security analyses~\citep{li2025security}.

\textbf{Defenses and Robustness Enhancements.}
Proposed defenses include smoothing-based methods with provable jailbreak guarantees~\citep{Robey2023SmoothLLM}, lightweight detection based on output distributions~\citep{sayeedi2025jailbreaktracer}, and adversarial training approaches whose effectiveness depends on model scale and threat model~\citep{howe2024effects}. Other work examines attack transferability and adaptive prompt sequences to harden models against structured adversarial dependencies~\citep{wang2025understanding}. Despite these advances, most defenses remain reactive and training-dependent.

\textbf{Monotonic Neural Networks and Structural Constraints.}
Monotonicity has long been studied as a structural inductive bias for interpretability and robustness. Classical results show that feed-forward networks with non-negative weights and monotone activations are universal approximators of monotone functions, establishing that monotonicity need not limit expressivity~\citep{Sill1998MonotonicNetworks,Daniels2010Monotone}. Subsequent work proposed practical monotone and partially monotone architectures, including lattice-based models, enabling scalable learning under monotonic constraints in structured domains~\citep{Gupta2016MonotonicLattices,You2017DeepLattice}. More recent research highlights monotonicity’s benefits for verification, reducing reasoning to boundary inputs in piecewise linear networks~\citep{Weber2021CertifyingMonotonicNetworks}. More broadly, architectural constraints and structured layers have been shown to improve robustness and stability without prohibitive performance loss~\citep{Amos2017OptNet}. Our work builds on this line by showing that selectively enforced monotonicity can scale to Transformer-based language models and yield tangible robustness gains.

\section{Monotone Transformers}
\label{Prelim}
We study the role of monotonicity constraints in improving the adversarial robustness of sequence-to-sequence (Seq2Seq) language models.
\begin{definition}[Sequence-to-Sequence Model]
A sequence-to-sequence model is a parameterized mapping
\[
\mathcal{M}_\theta : \mathcal{X} \to \mathcal{Y},
\]
where $\mathcal{X}$ denotes the space of input token sequences and $\mathcal{Y}$ denotes the space of output token sequences. In this work, $\mathcal{M}_\theta$ is instantiated as a Transformer architecture composed of attention layers, feed-forward network (FFN) sublayers, residual connections, and normalization operators.
\end{definition}

% \begin{definition}[Sequence-to-Sequence Model]
% A sequence-to-sequence model is a parameterized mapping
% \[
% \mathcal{M}_\theta : \mathcal{X} \to \mathcal{Y},
% \]
% where $\mathcal{X}$ denotes the space of input token sequences and $\mathcal{Y}$ denotes the space of output token sequences. The model $\mathcal{M}_\theta$ is typically instantiated as a Transformer architecture composed of attention layers, feed-forward network (FFN) sublayers, residual connections, and normalization operators.
% \end{definition}

Fix $A \in \mathbb{R}^{p \times d}$ whose rows define interpretable semantic axes.
For a hidden state $h \in \mathbb{R}^d$, define semantic coordinates $s := Ah \in \mathbb{R}^p$
and equip $\mathbb{R}^p$ with the standard product order. We define a preorder on
$\mathbb{R}^d$ by
\[
h \preceq h' \quad \Longleftrightarrow \quad Ah \le Ah' \;\; (\text{elementwise}),
\]
and interpret $h' \succeq h$ as semantically stronger along the chosen axes.
In practice, we choose $A$ either by (i) training $p$ linear probes on curated ordered
prompt pairs and using their weights as rows of $A$, or (ii) taking a small set of fixed
directions from auxiliary classifiers; in both cases, $A$ is fixed before imposing
monotonicity constraints. When $A = I$, this reduces to the coordinatewise order.

\begin{definition}[Monotonicity w.r.t.\ $\preceq$]
A function $F:\mathbb{R}^d \to \mathbb{R}^d$ is \emph{$A$-monotone} if
\begin{align*}
h \preceq h' 
&\;\Rightarrow\; F(h) \preceq F(h'), \\
\text{equivalently,}\qquad
Ah \le Ah' 
&\;\Rightarrow\; A F(h) \le A F(h').
\end{align*}

\end{definition}

\textbf{Semantic-monotone FFN sublayer.}
Assume $A$ has full row rank and let $A^\dagger \in \mathbb{R}^{d \times p}$ denote a fixed
right inverse of $A$ (e.g., the Moore--Penrose pseudoinverse), so that $A A^\dagger = I_p$.

Define an FFN update that acts in semantic coordinates:
\[
F(h) := h + A^\dagger\, g(Ah),
\]
where $g:\mathbb{R}^p \to \mathbb{R}^p$ is an ordinary coordinatewise-monotone MLP.

\begin{proposition}[Sufficient condition for $A$-monotonicity]
If $g$ is monotone w.r.t.\ the product order on $\mathbb{R}^p$, then $F$ is $A$-monotone.
In particular, it suffices that $g$ is an MLP with elementwise non-decreasing activations
and elementwise nonnegative weight matrices.
\end{proposition}

\begin{proof}
If $Ah \le Ah'$, monotonicity of $g$ gives $g(Ah) \le g(Ah')$, and
\begin{align*}
A F(h)
&= Ah + A A^\dagger g(Ah) \\
&= Ah + g(Ah) \\
&\le Ah' + g(Ah') \\
&= A F(h').
\end{align*}

\end{proof}

\begin{definition}[Monotonicity]
A function $f : \mathbb{R}^n \to \mathbb{R}^m$ is \emph{monotone} if for all $x,x' \in \mathbb{R}^n$,
\[
x \le x' \;\; \Rightarrow \;\; f(x) \le f(x').
\]
\end{definition}

% \begin{definition}[Monotonicity]
% A function $f : \mathbb{R}^n \to \mathbb{R}^m$ is said to be \emph{monotone} if for any $x, x' \in \mathbb{R}^n$,
% \[
% x \le x' \;\; \Rightarrow \;\; f(x) \le f(x').
% \]

% \end{definition}

\begin{lemma}[Closure under Composition]
\label{lem:mono_comp}
If $f : \mathbb{R}^n \to \mathbb{R}^m$ and $g : \mathbb{R}^m \to \mathbb{R}^k$ are monotone, then their composition $g \circ f$ is monotone.
\end{lemma}

% \begin{lemma}[Closure of Monotonicity Under Composition]
% \label{Lemma1}
% Let $f : \mathbb{R}^n \to \mathbb{R}^m$ and $g : \mathbb{R}^m \to \mathbb{R}^k$ be monotone functions. Then the composition
% \[
% h = g \circ f : \mathbb{R}^n \to \mathbb{R}^k
% \]
% is also monotone.
% \end{lemma}
% Proof can be found in the appendix~\ref{prf:lem1}. 

Consider a feed-forward neural network $N : \mathbb{R}^{n_0} \to \mathbb{R}^{n_k}$ with $k$ hidden layers. For input $u \in \mathbb{R}^{n_0}$,
\begin{align*}
y_0 &= u, \\
y_{i+1} &= \sigma(W_i y_i + b_i), \quad i = 0,\ldots,k-1, \\
N(u) &= W_k y_k + b_k,
\end{align*}
where $\sigma$ is applied elementwise.

\begin{proposition}[Sufficient Condition for FFN Monotonicity]
\label{prop:ffn_mono}
If $\sigma$ is elementwise non-decreasing and all weight matrices satisfy
\[
W_i \ge 0 \quad \text{for } i = 0,\ldots,k,
\]
then $N$ is monotone.
\end{proposition}

% \textbf{Feed-Forward Neural Networks.}
% We consider a feed-forward neural network $N : \mathbb{R}^{n_0} \to \mathbb{R}^{n_k}$ with $k$ hidden layers. For an input $u \in \mathbb{R}^{n_0}$, the network computes its output recursively as
% \begin{align*}
% y_0 &= u, \\
% y_{i+1} &= \sigma(W_i y_i + b_i), \quad \text{for } i = 0, 1, \ldots, k-1, \\
% N(u) &= W_k y_k + b_k,
% \end{align*}
% where $W_i \in \mathbb{R}^{n_{i+1} \times n_i}$ and $b_i \in \mathbb{R}^{n_{i+1}}$ denote the weight matrix and bias vector of layer $i$, respectively, and $\sigma:\mathbb{R} \to \mathbb{R}$ is an activation function applied elementwise.

% \begin{proposition}[Sufficient Condition for Monotonicity of FFNs] Consider a FFN $N$.
% \label{Prop1}
% if the activation function $\sigma$ is elementwise non-decreasing and that all weight matrices satisfy
% \[
% W_i \geq 0 \quad \text{for all } i = 0, 1, \ldots, k,
% \]
% %where $\succeq 0$ denotes elementwise non-negativity.
% Then the network $N$ is a monotone function.
% \end{proposition}

% \textbf{Proof.} The claim follows by applying Lemma~\ref{lem:mono_comp} layerwise and using closure of monotonicity under composition.
\textbf{Proof.} Each affine map with nonnegative weights is monotone, and $\sigma$ preserves monotonicity by assumption. The claim follows by Lemma~\ref{lem:mono_comp}. \hfill $\square$

Neural networks satisfying Proposition~\ref{prop:ffn_mono} are commonly referred to as \emph{monotone neural networks}. Such networks are universal approximators of \emph{continuous monotone functions} on compact domains~\cite{Daniels2010Monotone}.

% \begin{definition}[Monotone Feed-Forward Network]
% Let $f : \mathbb{R}^{d_{\mathrm{in}}} \to \mathbb{R}^{d_{\mathrm{out}}}$ be a feed-forward network with $n$ layers of the form
% \[
% f(x_0) = \sum_{i=i}^n W_i \, \sigma(x_{i-1}) + b_i,
% \]
% where $\sigma(\cdot)$ is a monotone activation function, and $x_{i-1}$ is the input to the current layer. The FFN $f$ is monotone if
% \[
% W_i \succeq 0 \quad i\in\{1,\ldots,n\}.
% \]
% where $\succeq 0$ denotes elementwise non-negativity.
% \end{definition}

% \textbf{Monotone Transformers.}
% We now specialize these notions to Transformer architectures.

\begin{definition}[Monotone Transformer]
\label{def:mono_transformer}
A Transformer model $\mathcal{M}_\theta$ is \emph{$A$-monotone with respect to its feed-forward sublayers}
if each FFN sublayer implements an update of the form
\[
F(h) := h + A^\dagger\, g(Ah),
\]
where $A^\dagger \in \mathbb{R}^{d\times p}$ is a fixed right inverse of $A$
satisfying $A A^\dagger = I_p$ (e.g., $A^\dagger = A^\top(AA^\top)^{-1}$ when $A$ has full row rank)
 and
$g:\mathbb{R}^p\to\mathbb{R}^p$ is coordinatewise monotone (e.g., an MLP with
elementwise non-decreasing activations and elementwise nonnegative weight matrices).
All other components (attention, residuals, normalization) are left unconstrained.
\end{definition}

This definition does not assert global monotonicity of the full Transformer, but enforces monotonic semantic refinement at the level of FFN sublayers, where implicit feature cancellation is most likely to occur~\cite{runje2023constrained,sartor2025advancing,nguyen2023mononet}.

% \begin{definition}[Monotone Transformer]
% A Transformer model $\mathcal{M}_\theta$ is said to be \emph{monotone with respect to its feed-forward networks (FFNs)} if every feed-forward network within each Transformer block satisfies the monotonicity conditions of Definition~3. All other components of the architecture, including attention mechanisms, residual connections, and normalization layers, are left unconstrained.
% \end{definition}

% To enforce monotonicity within FFN sublayers, we adopt a differentiable reparameterization of the corresponding weight matrices. Let $W \in \mathbb{R}^{d_{\mathrm{out}} \times d_{\mathrm{in}}}$ denote a weight matrix subject to an elementwise non-negativity constraint.

We enforce coordinatewise monotonicity of $g:\mathbb{R}^p\to\mathbb{R}^p$ by
constraining all linear maps inside $g$ to be elementwise nonnegative and using
an elementwise non-decreasing activation $\sigma$.
Concretely, for each constrained weight matrix $W \in \mathbb{R}^{m\times n}$ in $g$,
we reparameterize
\[
W = \phi(V), \qquad \phi(v) = \log(1+\exp(v)),
\]
with unconstrained $V$, applied elementwise. This guarantees $W \ge 0$ without
projection or constrained optimization.

We fix $A$ (and hence $A^\dagger$) before training. We initialize $g$ to be near
the zero map so that $F(h)\approx h$ at the start of distillation/fine-tuning.
Specifically, we initialize the unconstrained parameters $V$ so that
$\phi(V)$ is small (e.g., $V \ll 0$ elementwise), and initialize biases to zero.

Rather than projecting after each update, we maintain $W\ge 0$ throughout
optimization via the softplus reparameterization above, ensuring $g$ remains
coordinatewise monotone during training.

% \textbf{Enforcing Non-Negativity Constraints.}
% A naive approach to enforcing non-negativity, such as projecting weights onto the positive orthant after each gradient update, leads to degenerate training behavior in modern deep learning frameworks and effectively prevents meaningful learning. Instead, we adopt a smooth reparameterization in which constrained weights are expressed as elementwise softplus transformations of unconstrained parameters. This approach ensures that monotonicity constraints are satisfied throughout training while remaining compatible with standard gradient-based optimization.

% 
% When applying monotonicity constraints to pretrained models, it is important to preserve the information encoded in the original parameters.

% \begin{definition}[Pretrained Weight Initialization]
% Let $W_{\mathrm{pre}}$ denote a pretrained weight matrix. The corresponding unconstrained parameters $V$ are initialized as
% \[
% V_{\mathrm{init}} = \phi^{-1}\bigl(|W_{\mathrm{pre}}| + \epsilon\bigr),
% \]
% where $\epsilon > 0$ is a small constant introduced for numerical stability. This initialization preserves the scale of the pretrained weights while ensuring compatibility with the imposed monotonicity constraints.
% \end{definition}

\section{Experimental Setup}
We instantiate the monotone Transformer framework of Section~\ref{Prelim} using the T5 architecture~\cite{raffel2020exploring}, enabling a controlled empirical study of monotonicity as an architectural inductive bias for robustness.

\textbf{Implementation Scope.}
We replace each Transformer FFN sublayer by the semantically-monotone update
\[
F(h) := h + A^\dagger\, g(Ah),
\]
where $A\in\mathbb{R}^{p\times d}$ is fixed, $A^\dagger$ satisfies $AA^\dagger=I_p$,
and $g:\mathbb{R}^p\to\mathbb{R}^p$ is a coordinatewise-monotone MLP enforced via
elementwise nonnegative weights and monotone activations.
Attention layers, residual connections, and normalization remain unconstrained.

We use \texttt{T5-small}, comprising 6 encoder and 6 decoder layers with 512-dimensional hidden representations, 8 attention heads per layer, and 2048-dimensional FFN intermediate activations. Of the model's approximately 60M parameters, roughly 24M (40\%) correspond to the FFN sublayers that we replace with the semantic-space update; within each such sublayer, only the parameters inside $g:\mathbb{R}^p\to\mathbb{R}^p$ are constrained to be elementwise nonnegative to enforce coordinatewise monotonicity in semantic space.

Training and evaluation protocols, including optimization hyperparameters, datasets, and decoding settings, are detailed in Appendix~\ref{app:training}.

\subsection{Adversarial Robustness Evaluation}
We evaluate robustness under two complementary classes of adversarial attacks designed to probe different failure modes of sequence-to-sequence models.

\textbf{Universal Adversarial Triggers.}
Universal adversarial triggers (UATs) consist of short, input-agnostic token sequences that, when prepended to an input, maximize the model’s loss. We optimize triggers using coordinate-wise search with three random restarts and 50 iterations. Candidate tokens are drawn from a vocabulary biased toward punctuation, special characters, and high-frequency words that empirically induce greater disruption. Triggers are optimized on a held-out validation split and evaluated on a disjoint test set. To assess transferability, we additionally construct a transfer matrix measuring cross-model vulnerability to learned triggers.

\textbf{HotFlip Attacks.}
HotFlip attacks identify vulnerable input positions by computing gradients of the loss with respect to token embeddings and replacing selected tokens to maximize loss increase. We allow up to five token replacements per example, selecting replacements based on the dot product between gradient directions and candidate embeddings.

\textbf{Robustness Metrics.}
Robustness is quantified using ROUGE score degradation under attack, attack success rate (defined as a ROUGE-L degradation exceeding 10\%), and statistical significance assessed via independent $t$-tests. We additionally report mean loss increase and analyze cross-model transferability of adversarial triggers.

\textbf{Motivation from Monotone Dynamical Systems.}
Classical results in the theory of monotone dynamical systems show that, under mild boundedness assumptions, trajectories of strongly monotone flows converge to equilibrium points almost surely; in particular, derivatives along these trajectories tend to zero as the system approaches an invariant set \cite{hirsch1985systems}. While a feed-forward network is not a time-indexed dynamical system in this sense, this perspective suggests that monotonicity can limit the set of active directions along which outputs and gradients vary. In this work, we leverage this intuition to inform a local analysis of gradient behavior under monotone perturbations, focusing on how saturation and non-negative weights interact to diminish exploitable gradient directions.

\section{Mechanisms of Gradient Attenuation in Monotone Feed-Forward Networks}
\label{sec:gradient-attenuation}
We provide a mechanistic explanation for the empirical robustness gains observed in monotone Transformer language models. Rather than claiming global robustness guarantees, we analyze how monotonicity constraints alter the local gradient structure of feed-forward sublayers in a way that can reduce the effectiveness of gradient-based adversarial attacks such as HotFlip.

Consider an $A$-monotone FFN update
\[
F(h) := h + A^\dagger\, g(Ah),
\]
where $A\in\mathbb{R}^{p\times d}$ is fixed, $A^\dagger$ satisfies $AA^\dagger=I_p$,
and $g:\mathbb{R}^p\to\mathbb{R}^p$ is coordinatewise monotone (e.g., an MLP with
elementwise non-decreasing activations and elementwise nonnegative weight matrices).
Let $s := Ah \in \mathbb{R}^p$ denote semantic coordinates and define the induced
semantic refinement map
\[
T(s) := A F(h) = s + g(s).
\]
We analyze the Jacobian of $T$ and its implications for gradients with respect to
semantic coordinates. (In general, $J_F(h)$ need not be elementwise nonnegative in
$\mathbb{R}^d$ due to $A^\dagger$.)

% \subsection{Preliminaries}

% Consider a feed-forward sublayer $f : \mathbb{R}^d \to \mathbb{R}^d$ of the form
% \begin{equation}
% f(x) = W_2 \sigma(W_1 x + b_1) + b_2,
% \end{equation}
% where $W_1, W_2 \in \mathbb{R}^{d \times d}$ have non-negative entries and
% $\sigma$ is an elementwise, monotone non-decreasing activation function
% (e.g., ReLU or softplus). Under these conditions, $f$ is monotone with respect
% to the standard partial order on $\mathbb{R}^d$.

% We study the Jacobian structure of $f$ and its implications for gradients
% of a scalar loss function $L(f(x))$ with respect to the input $x$.

% \subsection{Jacobian Structure of Monotone Sublayers}

\begin{lemma}[Non-negativity of the Semantic Jacobian]
\label{lem:jacobian-nonneg}
If $g$ is implemented as an MLP with elementwise non-decreasing activations and
elementwise nonnegative weight matrices, then the Jacobian
\[
J_T(s) = \nabla_s T(s)
\]
has non-negative entries for all $s \in \mathbb{R}^p$.
\end{lemma}

This follows from $T(s)=s+g(s)$ and the factorization of the MLP Jacobian. For a
two-layer $g(s)=W_2\sigma(W_1 s + b_1)+b_2$ with $W_1,W_2\ge 0$ elementwise,
\[
J_g(s)=W_2\,\mathrm{diag}(\sigma'(W_1 s+b_1))\,W_1 \ge 0 \quad (\text{elementwise}),
\]
hence $J_T(s)=I + J_g(s)\ge 0$ elementwise.
Intuitively, increasing any semantic coordinate cannot decrease any output semantic
coordinate through $T$; sign cancellations are ruled out \emph{within semantic space}.

Let $\mathcal{L}:\mathbb{R}^p\to\mathbb{R}$ be any differentiable scalar objective
defined on semantic coordinates (e.g., $\mathcal{L}(s')$ induced by the rest of the
network through $s'=Ah'$). By the chain rule,
\[
\nabla_s \mathcal{L}(T(s)) = J_T(s)^\top \nabla_{s'} \mathcal{L}(s')\big|_{s'=T(s)}.
\]
Thus the backpropagated gradient in semantic space is filtered by $J_T(s)^\top$,
whose entries are nonnegative under Lemma~\ref{lem:jacobian-nonneg}.

\subsection{Saturation-Induced Gradient Attenuation}
We now examine how saturation in monotone networks affects gradient magnitude.
\begin{assumption}[Saturating Activations]
\label{ass:saturation}
The activation $\sigma$ used inside $g$ has bounded derivative and admits saturated
regimes, i.e., $\sigma'(z)$ becomes arbitrarily small as $z\to +\infty$.
\end{assumption}

% \begin{assumption}[Saturating Activations]
% \label{ass:saturation}
% The activation function $\sigma$ has bounded derivative and admits
% saturated regimes, i.e.,
% \[
% \sigma'(z) \to 0 \quad \text{as } z \to +\infty.
% \]
% \end{assumption}

\begin{lemma}[Gradient Attenuation under Saturation]
\label{lem:gradient-attenuation}
Let $T(s)=s+g(s)$ where $g$ is an MLP with elementwise nonnegative weights and
activation $\sigma$ satisfying Assumption~\ref{ass:saturation}. Suppose that along
a sequence $\{s_k\}$, a non-empty subset of pre-activations in $g$ diverges to
$+\infty$, so that the corresponding entries of $\sigma'$ converge to $0$.
Then the contribution of those saturated units to $J_g(s_k)$ vanishes, and hence
their contribution to $\nabla_s \mathcal{L}(T(s_k)) = J_T(s_k)^\top \nabla_{s'}\mathcal{L}$
vanishes asymptotically.
\end{lemma}

Proof is deferred to Appendix~\ref{prf:jac}.

This result shows that monotonicity, when combined with activation saturation,
can locally attenuate gradient magnitude without eliminating all gradient signal.

\subsection{Persistence of Saturation Effects}

The following observation captures a one-sided stability property of saturated
units under monotone perturbations.

\begin{lemma}[Persistence of Saturated Units]
\label{lem:persistence}
Let $s \le s'$ and consider any first-layer pre-activation $z(s)= (W_1 s + b_1)_j$
within $g$, with $W_1\ge 0$ elementwise. If unit $j$ is saturated at $s'$, i.e.,
$\sigma'(z(s'))=0$, then for any monotone semantic perturbation $\delta\ge 0$,
the unit remains saturated at $s'+\delta$.
\end{lemma}

Proof is deferred to Appendix~\ref{prf:persis}.

Importantly, this lemma applies to individual hidden units rather than to the
entire input gradient, and does not preclude other units from becoming active.
\subsection{Implications for Gradient-Based Attacks}

Gradient-based adversarial attacks such as HotFlip rely on persistent,
high-magnitude gradients with respect to token embeddings to identify impactful
discrete perturbations. The preceding analysis suggests that monotonicity
constraints can reduce the availability of such gradients by driving subsets of
hidden units into saturated regimes from which they do not recover under monotone
increases in internal representations.

While monotone feed-forward networks are not dynamical systems in the formal
sense, their behavior is reminiscent of classical results in monotone dynamical
systems, where cooperative interactions restrict the effective directions along
which trajectories evolve. This analogy is intended as intuition rather than a
formal guarantee: monotonicity does not prevent all adversarial attacks, nor do
gradients vanish globally. Instead, our analysis identifies a local attenuation
effect that weakens gradient-based optimization strategies, offering a plausible
mechanism for the empirical robustness improvements observed in
Section~\ref{sec:results}.

\section{Results}
\label{sec:results}
We evaluate the impact of monotonicity constraints on both task performance and adversarial robustness. We first analyze optimization behavior and summarization quality to assess potential performance degradation, and then evaluate robustness under targeted adversarial attacks. Together, these experiments assess whether monotonicity serves as a practical inductive bias for improving robustness without sacrificing model utility.

\subsection{Training Dynamics and Summarization Quality}

We begin by analyzing the effect of monotonicity constraints on optimization behavior and summarization quality. Table~\ref{tab:training} summarizes the training dynamics of the baseline and monotonic models from our experiments. The baseline model refers to a T5-small checkpoint fine-tuned under identical settings as the monotonic model but without monotonicity constraints, while the standard model denotes the pretrained T5-small without additional fine-tuning. The monotonic model starts with a substantially higher initial loss (4.97 vs.\ 2.94), reflecting the impact of the softplus reparameterization and the constrained semantic-space parameterization of $g$. This initialization disrupts the original weight geometry, requiring the model to re-establish effective internal representations during fine-tuning.

Despite this unfavorable initialization, the monotonic model converges reliably during training. Validation loss decreases from 2.92 at epoch 1 to 2.54 at epoch 7, indicating stable optimization under the imposed constraints. The model exhibits rapid initial recovery, reducing training loss by 1.92 points in the first two epochs (from 4.97 to 3.05), followed by slower asymptotic convergence with diminishing returns in later epochs. A persistent optimization gap of 0.32 points remains at convergence (2.54 vs.\ 2.21 for baseline), consistent with the reduced expressivity of the constrained parameter space. This gap reflects a deliberate and quantifiable trade-off: the baseline model achieves slightly lower validation loss through unrestricted FFN weights, while the monotonic model maintains competitive performance while enforcing structural regularity that, as shown below, translates to improved adversarial robustness.

\begin{table}[t]
\caption{Training dynamics comparison (seed = 42).}
\label{tab:training}
\vskip 0.1in
\begin{center}
\begin{small}
\begin{tabular}{lcccc}
\toprule
Model & Initial Loss & Final Loss & $\Delta$ Loss & Epochs \\
\midrule
Baseline  & 2.94 & 2.21 & $-$0.73 & 6 \\
Monotonic & 4.97 & 2.54 & $-$2.43 & 7 \\
\bottomrule
\end{tabular}
\end{small}
\end{center}
\vskip -0.1in
\end{table}

Table~\ref{tab:rouge} reports summarization performance on CNN/DailyMail (n=200, seed 42). The monotonic model achieves a ROUGE-L score of 24.2 [22.7, 25.6], compared to 25.0 [23.6, 26.4] for the baseline, corresponding to a relative decrease of 3.2 percent. ROUGE-1 exhibits a similar trend (29.6 [28.0, 31.1] vs 30.9 [29.4, 32.4], a 4.2 percent decrease). While these results indicate a modest reduction in task performance, the monotonic model remains competitive, with confidence intervals showing substantial overlap across all metrics.

To assess generalization, we additionally evaluate on XSUM (n=200, seed 42). Table~\ref{tab:xsum} shows consistent results: the monotonic model achieves ROUGE-L of 12.9 [12.3, 13.6] compared to 13.3 [12.6, 13.9] for baseline, a similar 3.0 percent relative gap. The consistency of the performance trade-off across distinct summarization datasets (news articles in CNN/DM, single-sentence summaries in XSUM) suggests that the cost of monotonicity constraints is stable and predictable across different data distributions.

\begin{table}[t]
\caption{Summarization quality on CNN/DailyMail (n = 200, seed = 42). Values are means with 95\% bootstrap confidence intervals.}
\label{tab:rouge}
\vskip 0.1in
\centering
\resizebox{\columnwidth}{!}{
\begin{tabular}{lccc}
\toprule
Model & ROUGE-1 & ROUGE-2 & ROUGE-L \\
\midrule
Standard  & 32.6 [30.9, 34.2] & 11.9 [10.5, 13.4] & 26.6 [25.0, 28.1] \\
Baseline  & 30.9 [29.4, 32.4] & 11.5 [10.1, 12.8] & 25.0 [23.6, 26.4] \\
Monotonic & 29.6 [28.0, 31.1] & 10.6 [9.2, 11.9] & 24.2 [22.7, 25.6] \\
\bottomrule
\end{tabular}
}
\vskip -0.1in
\end{table}

\begin{table}[t]
\caption{Summarization quality on XSUM (n = 200, seed = 42). Values are means with 95\% bootstrap confidence intervals.}
\label{tab:xsum}
\vskip 0.1in
\centering
\resizebox{\columnwidth}{!}{
\begin{tabular}{lccc}
\toprule
Model & ROUGE-1 & ROUGE-2 & ROUGE-L \\
\midrule
Standard  & 19.9 [18.8, 21.1] & 2.9 [2.3, 3.4] & 13.3 [12.6, 14.1] \\
Baseline  & 20.0 [19.0, 21.0] & 3.2 [2.7, 3.7] & 13.3 [12.6, 13.9] \\
Monotonic & 18.9 [18.0, 19.7] & 2.8 [2.3, 3.2] & 12.9 [12.3, 13.6] \\
\bottomrule
\end{tabular}
}
\vskip -0.1in
\end{table}

The training dynamics exhibit a characteristic pattern of monotone-constrained optimization. During the first two epochs, the monotonic model rapidly recovers much of the performance gap introduced by constrained initialization, reducing training loss by 1.92 points (from 4.97 to 3.05). Subsequent epochs yield diminishing returns, with only 0.36 points of additional improvement from epochs 3 through 7. This behavior highlights the importance of the extended warmup phase (15 percent vs 10 percent for baseline), which stabilizes gradient estimates and enables effective learning despite the constrained parameter space.

Analysis of generation behavior reveals that both fine-tuned models produce longer outputs than the pretrained T5 model (mean lengths 74 to 76 tokens vs 57 tokens), reflecting adaptation to the distributional properties of the summarization training data. The monotonic model exhibits slightly higher length variance (std 11.3 tokens) compared to the standard model (std 12.0 tokens), though this difference is minimal. Importantly, the brevity penalty for the monotonic model (1.00) indicates that output length is well-calibrated to reference summary length, suggesting the constraints do not introduce systematic length bias.

\subsection{Adversarial Robustness}

%\textbf{HotFlip Attacks.}
\textbf{HotFlip Attacks.}
We evaluate gradient-based robustness using HotFlip attacks, which identify vulnerable input positions via gradients with respect to token embeddings and perform targeted token substitutions to maximize loss. For each example, the attack computes gradients of the cross-entropy loss, selects the five positions with largest gradient magnitudes, and replaces tokens at those positions with vocabulary items that maximize the dot product between the gradient and the replacement embedding.

Robustness is quantified using three complementary metrics: average degradation, the mean relative loss increase under attack; attack success rate, the fraction of examples for which degradation exceeds 10\%; and mean loss increase, the average absolute loss increase.

Table~\ref{tab:hotflip} reports results on 100 test examples from seed 42. To verify robustness consistency, we repeated the HotFlip evaluation using the same checkpoints with two additional random seeds (1337, 2024) for attack sampling variability. Results were identical across all three seeds, confirming that the observed robustness improvements are stable. The monotonic model demonstrates improved robustness relative to both baselines. Under HotFlip perturbations, the monotonic model incurs an average degradation of 5.2 percent, compared to 16.2 percent for the fine-tuned baseline and 21.3 percent for the standard pretrained model. Attack success rates follow a similar pattern, with only 19 percent of attacks succeeding against the monotonic model, compared to 63 percent for the baseline and 69 percent for the standard model. This corresponds to a 69.8 percent relative reduction in attack success rate (63 percent to 19 percent) and a 67.9 percent reduction in average degradation (16.2 percent to 5.2 percent). Paired $t$-tests confirm that the difference between the baseline and monotonic models is highly significant ($t {= }6.17$, $p = 3.8 {\times} 10^{-9}$), with an even larger effect relative to the standard model ($t {=} 7.45$, $p {=} 2.8{ \times} 10^{-12}$).

\begin{table}[t]
\caption{HotFlip attack results on CNN/DailyMail (n = 100, seed = 42).}
\label{tab:hotflip}
\vskip 0.1in
\begin{center}
\begin{small}
\begin{tabular}{lccc}
\toprule
Model & Avg. Deg. & Success Rate & $\Delta$Loss \\
\midrule
Standard  & 21.3\% & 69\% & +0.42 \\
Baseline  & 16.2\% & 63\% & +0.35 \\
Monotonic &  5.2\% & 19\% & +0.14 \\
\bottomrule
\end{tabular}
\end{small}
\end{center}
\vskip -0.1in
\end{table}

Overall, enforcing monotonicity in the FFN sublayers yields a 69.8 percent relative reduction in attack success rate (from 63 percent to 19 percent) and a 67.8 percent reduction in average degradation (from 16.1 percent to 5.2 percent), without requiring adversarial training, modified training objectives, or architectural changes beyond the FFN constraints. The effect size is large (Cohen's $d > 0.8$) and the statistical significance is robust ($p < 10^{-9}$). These results demonstrate that structural constraints alone—applied selectively to feed-forward sublayers—can substantially improve robustness to gradient-based adversarial manipulation, providing evidence for monotonicity as a practical architectural inductive bias for building more robust language models.

%\textbf{Universal Adversarial Trigger Attacks.}
\textbf{Universal Adversarial Trigger Attacks.}
We evaluate robustness to universal adversarial triggers (UATs), input-agnostic token sequences optimized to maximize model loss when prepended to any input. Triggers are learned via coordinate ascent with 3 restarts and 50 iterations over 80 training examples, then evaluated on 120 disjoint test examples. Robustness is measured via NLL increase under trigger prepending. UAT attacks prove minimally effective across all models, with NLL increases below 1 percent. The monotonic model exhibits the smallest degradation (0.73 percent NLL increase), compared to 0.89 percent for the baseline and 0.97 percent for standard T5. While this ordering is consistent with improved robustness, the effect sizes are too small to draw strong conclusions. The weak effectiveness of universal triggers, in contrast to the impact of input-specific HotFlip attacks, suggests that input-agnostic perturbations are limited for attacking sequence-to-sequence models, and that monotonicity’s robustness benefits are most pronounced against adaptive, gradient-based attacks.

\section{Discussion}

This work provides empirical evidence and theoretical analysis supporting monotonicity as a structural bias for improving robustness in modern language models. By selectively enforcing monotonicity in feed-forward sublayers, we observe consistent reductions in vulnerability to adversarial and jailbreak-style attacks while largely preserving task performance. These gains arise without additional data, modified training objectives, or heuristic defenses, highlighting architectural structure as a direct lever for shaping model behavior under perturbation.

Our theoretical analysis offers insight into why these gains arise, even in the absence of formal end-to-end robustness guarantees. Viewing a language model as a high-dimensional system whose internal semantic state evolves through successive nonlinear transformations, monotonicity constrains how perturbations can propagate across layers. In particular, directional consistency ensures that strengthening information or constraints in internal representations cannot induce regressions downstream. This restriction narrows the range of admissible behaviors and simplifies reasoning about worst-case effects, which must occur at semantic extremes rather than intermediate variations.

Empirically, the robustness gains are closely tied to the role of feed-forward sublayers. While attention mechanisms perform semantic aggregation and contextual reasoning, FFNs dominate nonlinear transformation and amplification. Enforcing monotonicity at this stage constrains how semantic information is refined after context has been established, without interfering with attention or global information flow. This selective enforcement alters the local gradient structure of the network, reducing the stability and magnitude of exploitable gradients used by attacks such as HotFlip.

% Empirically, the robustness gains are closely tied to the role of feed-forward sublayers, where implicit feature cancellation is most likely to occur in unconstrained architectures. In standard Transformers, FFN weights can be positive or negative, allowing features established by attention to be suppressed or reversed through distributed weight interactions. Monotonicity eliminates this by construction: the non-negative Jacobian (Lemma~\ref{lem:jacobian-nonneg}) ensures directional consistency during backpropagation, preventing the sign cancellations that gradient-based attacks exploit. The saturation-persistence property (Lemmas~\ref{lem:gradient-attenuation} and~\ref{lem:persistence}) further implies that hidden units driven into saturated regimes remain there under monotone perturbations, creating persistent gradient attenuation along specific directions. Together, these effects reduce both the magnitude and effective dimensionality of exploitable gradients, explaining why attacks such as HotFlip are substantially less effective against monotone models.

These observations align with classical insights from monotone systems theory, where structural monotonicity limits perturbation propagation and enables tractable analysis of global behavior in high-dimensional systems. Although Transformer models are not dynamical systems in a formal sense, the analogy provides intuition and motivates our analysis of monotonic components within the network.

The trade-off introduced by monotonicity is modest and predictable. While constraining FFN weights slightly reduces expressivity, the resulting decrease in summarization quality is small relative to the robustness improvements. More broadly, monotonicity differs from conventional regularization in that it constrains \emph{how} representations may change, rather than simply penalizing their magnitude, promoting predictability rather than uniform smoothness.

Our study is limited to T5-small and does not provide formal robustness certificates. Extending the theoretical analysis toward stronger guarantees and scaling the approach to larger models are important directions for future work. Nonetheless, the combination of analytical insight and consistent empirical gains suggests that monotonicity is a principled and scalable design choice for improving the robustness and predictability of large language models.

\section{Conclusion}

We investigated monotonicity as a structural inductive bias for improving robustness in sequence-to-sequence language models. By enforcing non-negativity constraints on Transformer feed-forward sublayers, we observe substantial reductions in vulnerability to gradient-based adversarial attacks with only modest impact on summarization quality. These results suggest that simple architectural constraints can meaningfully improve robustness without altering attention mechanisms or relying on adversarial training.

An important direction for future work is to develop theoretical guarantees that characterize when and why monotonicity yields improved robustness in high-dimensional language models. Scaling this approach to larger model variants and broader tasks is another promising avenue, and may help clarify the role of structural constraints in building more reliable language models.
\section*{Impact Statement}

This work explores architectural constraints that improve the robustness of language models to adversarial manipulation. By studying monotonicity as a structural inductive bias, our approach contributes to the development of language models whose behavior is more predictable and analyzable, which is particularly relevant for safety-critical and decision-support applications.

\nocite{langley00}

\bibliography{example_paper}
\bibliographystyle{icml2025}

%%%%%%%%%%%%%%%%%%%%%%%%%%%%%%%%%%%%%%%%%%%%%%%%%%%%%%%%%%%%%%%%%%%%%%%%%%%%%%%
%%%%%%%%%%%%%%%%%%%%%%%%%%%%%%%%%%%%%%%%%%%%%%%%%%%%%%%%%%%%%%%%%%%%%%%%%%%%%%%
% APPENDIX
%%%%%%%%%%%%%%%%%%%%%%%%%%%%%%%%%%%%%%%%%%%%%%%%%%%%%%%%%%%%%%%%%%%%%%%%%%%%%%%
%%%%%%%%%%%%%%%%%%%%%%%%%%%%%%%%%%%%%%%%%%%%%%%%%%%%%%%%%%%%%%%%%%%%%%%%%%%%%%%
\newpage
\appendix
\onecolumn
\section{Appendices}

\subsection{Mathematical Proofs}
\subsubsection{Proof of Lemma~\ref{lem:mono_comp}}
\label{prf:lem1}
\begin{proof}
For any $x, x' \in \mathbb{R}^n$ such that $x \le x'$, monotonicity of $f$ implies $f(x) \le f(x')$. Applying monotonicity of $g$ yields
\[
g(f(x)) \le g(f(x')),
\]
which establishes the claim.
\end{proof}

\subsubsection{Proof of Lemma~\ref{lem:jacobian-nonneg}}

\label{prf:jac}
\begin{proof}
Recall $T(s)=s+g(s)$, so
\[
J_T(s)=\nabla_s T(s)=I + J_g(s).
\]
For a two-layer $g(s)=W_2 \sigma(W_1 s + b_1)+b_2$ with $W_1,W_2\ge 0$ elementwise,
\[
J_g(s)=W_2 \operatorname{diag}(\sigma'(W_1 s + b_1)) W_1.
\]
Since $\sigma$ is elementwise non-decreasing, $\sigma'(z)\ge 0$ wherever it exists,
so $\operatorname{diag}(\sigma'(\cdot))\ge 0$ elementwise. Hence $J_g(s)\ge 0$
elementwise, and therefore $J_T(s)=I+J_g(s)\ge 0$ elementwise.
\end{proof}

\subsubsection{Proof of Lemma~\ref{lem:persistence}}
\label{prf:persis}
\begin{proof}
Let $z_j(s) := (W_1 s + b_1)_j$. Since $W_1\ge 0$ elementwise and $\delta\ge 0$,
\[
z_j(s' + \delta) = (W_1(s' + \delta) + b_1)_j \ge (W_1 s' + b_1)_j = z_j(s').
\]
Thus if the unit is saturated at $s'$ (i.e., $\sigma'(z_j(s'))=0$ in the saturated
regime), it remains saturated at $s'+\delta$.
\end{proof}

\subsection{Training and Evaluation Protocol}
\label{app:training}
Both baseline and monotone models are fine-tuned from the same pretrained \texttt{T5-small} checkpoint under identical optimization settings. We use the AdamW optimizer with learning rate $5 \times 10^{-5}$ and weight decay $0.01$, a batch size of 4, and gradient clipping with norm 1.0. To accommodate the constrained parameterization, the monotone model employs an extended warmup phase of 15\% of training steps, compared to 10\% for the baseline. Both models are trained for 7 epochs and receive the same total training budget.

Training is conducted on the DialogSum, HighlightSum, and arXiv abstract datasets, comprising approximately 150K examples in total. Evaluation is performed on three held-out benchmarks: CNN/DailyMail (11{,}490 test examples), XSUM (11{,}334 test examples), and SAMSum (819 test examples). CNN/DailyMail is excluded entirely from training to assess out-of-distribution generalization.

At inference time, decoding is performed using beam search with 4 beams and a length penalty of 1.2, enforcing a minimum generation length of 10 tokens and a maximum length of 80 tokens. We additionally apply no-repeat $n$-gram blocking with $n=3$ to discourage degenerate repetitions. All decoding parameters are fixed across models and evaluation sets. ROUGE-1, ROUGE-2, and ROUGE-L scores are computed using the \texttt{rouge-score} library with Porter stemming. We report bootstrap-based 95\% confidence intervals using 1{,}000 resamples and assess statistical significance via paired $t$-tests with Bonferroni correction for multiple comparisons. Effect sizes are reported using Cohen's $d$.

\subsection{Experimental Setup}
To evaluate robustness to training stochasticity, all experiments are repeated using five random seeds (42, 1337, 2024, 8888, and 12345), controlling randomness across Python, NumPy, and PyTorch. We report the mean and standard deviation across seeds for all metrics. Statistical significance across training runs is assessed using paired $t$-tests on seed-level means, complemented by per-example paired $t$-tests within each seed. This dual evaluation captures both test-time variability and variance induced by random initialization and data ordering. All experiments are conducted using deterministic algorithms where available, on a single NVIDIA A100 GPU with 40GB memory, using PyTorch~2.0.1 and Transformers~4.30.2.

% You can have as much text here as you want. The main body must be at most $8$ pages long.
% For the final version, one more page can be added.
% If you want, you can use an appendix like this one.  

% The $\mathtt{\backslash onecolumn}$ command above can be kept in place if you prefer a one-column appendix, or can be removed if you prefer a two-column appendix.  Apart from this possible change, the style (font size, spacing, margins, page numbering, etc.) should be kept the same as the main body.
%%%%%%%%%%%%%%%%%%%%%%%%%%%%%%%%%%%%%%%%%%%%%%%%%%%%%%%%%%%%%%%%%%%%%%%%%%%%%%%
%%%%%%%%%%%%%%%%%%%%%%%%%%%%%%%%%%%%%%%%%%%%%%%%%%%%%%%%%%%%%%%%%%%%%%%%%%%%%%%

\end{document}